%% file: ijcai2020.tex
\documentclass{article}

\usepackage{ijcai20}

\usepackage{supertabular,booktabs}
\usepackage[utf8]{inputenc} % allow utf-8 input
\usepackage[T1]{fontenc}    % use 8-bit T1 fonts
\usepackage{url}            % simple URL typesetting
\usepackage{booktabs}       % professional-quality tables
\usepackage[thinlines]{easytable}
\usepackage{amsfonts}       % blackboard math symbols
\usepackage{nicefrac}       % compact symbols for 1/2, etc.
\usepackage{microtype}      % microtypography
\usepackage{amsmath}
\usepackage[title]{appendix}
\usepackage{amssymb}
\usepackage{float}
\usepackage{soul}
\usepackage[font=small]{caption}
\usepackage{graphicx}
\usepackage{enumitem}
\usepackage[hidelinks]{hyperref}
\usepackage{etoolbox}
\usepackage{tablefootnote}
\usepackage{algorithm}
\usepackage{algorithmic}
\usepackage{amssymb}
\usepackage{amsthm}
\usepackage{footnote}
\usepackage{longtable}
\usepackage{stackengine}

\usepackage{setspace}
\usepackage{subcaption}
\usepackage{color}
\usepackage{times}

\newcommand{\norm}[1]{\left\lVert#1\right\rVert}
\newcommand{\idot}[1]{\langle#1\rangle}
\def\delequal{\mathrel{\ensurestackMath{\stackon[1pt]{=}{\scriptstyle\Delta}}}}
\DeclareMathOperator{\tr}{tr}
\input{math_commands}

\title{Gradient Perturbation is Underrated for \\Differentially Private Convex Optimization}

\author{
Da Yu$^1$\thanks{Equal contribution.}\footnote{The work was done when this author was an intern at Microsoft Research Asia.} \and
Huishuai Zhang$^2$\footnotemark[1]\and
Wei Chen$^{2}$\and
Jian Yin$^1$ \And
Tie-Yan Liu$^2$\\
\affiliations
$^1$The School of Data and Computer Science, Sun Yat-sen University.\\Guangdong Key Laboratory of Big Data Analysis and Processing,  Guangzhou 510006, P.R.China\\
$^2$Microsoft Research Asia, Beijing, China\\
\emails
\{yuda3@mail2, issjyin@mail\}.sysu.edu.cn,
\{huishuai.zhang, wche, tie-yan.liu\}@microsoft.com
}
\begin{document}
\maketitle

\input{intro}

\input{mainresult}
\input{numerical}

\bibliography{ijcai20}
\bibliographystyle{named}

\clearpage

\onecolumn
\end{document}

%% file: math_commands.tex
\usepackage{amsmath,amsfonts,bm}

\def\eqref#1{equation~\ref{#1}}
\def\1{\bm{1}}

\def\vg{{\bm{g}}}

\def\vs{{\bm{s}}}

\def\vv{{\bm{v}}}

\def\vx{{\bm{x}}}
\def\vy{{\bm{y}}}
\def\vz{{\bm{z}}}

\def\mH{{\bm{H}}}
\def\mI{{\bm{I}}}

\def\mS{{\bm{S}}}

\DeclareMathAlphabet{\mathsfit}{\encodingdefault}{\sfdefault}{m}{sl}
\SetMathAlphabet{\mathsfit}{bold}{\encodingdefault}{\sfdefault}{bx}{n}

\def\sR{{\mathbb{R}}}

\newcommand{\nn}{\nonumber}
\newcommand{\E}{\mathbb{E}}

\newtheorem{definition}{\textbf{Definition}}

\newtheorem{lemma}{\textbf{Lemma}}
\newtheorem{theorem}{\textbf{Theorem}}

\newtheorem{remark}{\textbf{Remark}}

\newtheorem{example}{\textbf{Example}}

\DeclareMathOperator*{\argmin}{arg\,min}

%% file: intro.tex
\begin{abstract}

Gradient perturbation, widely used for differentially private optimization, injects noise at every iterative update to guarantee differential privacy. Previous work first determines the noise level that  can satisfy the privacy requirement and then analyzes the utility  of noisy gradient updates as in the non-private case.  In contrast, we explore how the privacy noise affects the optimization property. We show that for differentially private convex optimization, the utility guarantee of differentially private (stochastic) gradient descent is determined by an \emph{expected curvature} rather than the minimum curvature. The \emph{expected curvature}, which represents the average curvature over the optimization path, is usually much larger than the minimum curvature. By using the \emph{expected curvature}, we show that gradient perturbation can achieve a significantly improved utility guarantee that can theoretically justify the advantage of gradient perturbation over other perturbation methods. Finally, our extensive experiments suggest that  gradient perturbation with the advanced composition method indeed outperforms other perturbation approaches by a large margin, matching our theoretical findings.
\end{abstract}

\section{Introduction}
Machine learning has become a powerful tool for many applications.  The training process often needs access to private datasets, e.g., applications in financial and medical fields. Recent work has shown that the model learned from training data may leak unintended  information of individual records  \cite{shokri2017membership,hitaj2017deep}. It is known that \emph{Differential Privacy (DP)}  \cite{dwork2006our} is a golden standard for privacy preserving data analysis, which has also been deployed into real world applications \cite{erlingsson2014rappor,abowd2016challenge}. It provides provable privacy guarantee by ensuring that the influence of any individual record is negligible.

We study the fundamental problem when differential privacy meets machine learning: the \emph{differentially private  empirical risk minimization (DP-ERM)} problem  \cite{chaudhuri2011differentially,bassily2014differentially,talwar2015nearly,wu2017bolt,zhang2017efficient,wang2017differentially,jayaraman2018distributed,feldman2018privacy,iyengar2019towards,wang2019differentially}. DP-ERM minimizes the empirical risk while guaranteeing the output of the learning algorithm differentially private with respect to the training data.  Such privacy guarantee provides strong protection against potential adversaries  \cite{hitaj2017deep,rahman2018membership}. In order to guarantee privacy, it is necessary to introduce randomness to the learning algorithm. There are usually three ways to introduce randomness  according to the time of adding noise: \emph{output perturbation, objective perturbation} and \emph{gradient perturbation}.

\emph{Output perturbation} \cite{wu2017bolt,zhang2017efficient} first runs the learning algorithm the same as in the non-private case then adds noise to the output parameter. \emph{Objective perturbation} \cite{chaudhuri2011differentially,iyengar2019towards} perturbs the objective (i.e., the empirical loss)  then releases the (approximated) minimizer of the perturbed objective. \emph{Gradient perturbation} \cite{bassily2014differentially,abadi2016deep,jayaraman2018distributed}  perturbs  each intermediate update. 

Gradient perturbation comes with several advantages over output/objective perturbations. Firstly, gradient perturbation does not require strong assumption on the objective because it  only needs to bound the sensitivity of each gradient update rather than the whole learning process.  Secondly, gradient perturbation can release the noisy  gradient at each iteration without damaging the privacy guarantee as differential privacy is immune to \emph{post processing}  \cite{algofound}. Therefore, it is a favorable choice in the distributed optimization or the federated learning setting  \cite{konevcny2015federated,agarwal2018cpsgd,jayaraman2018distributed}. At last, based on our own experiments and the observation in \cite{lee2018concentrated}, gradient perturbation achieves much better empirical utility than output/objective perturbations for DP-ERM. 

However, the existing theoretical utility guarantee for gradient perturbation is the same as or strictly inferior to that of output/objective perturbation approaches as shown in Table~\ref{tb:utility}. This motivates us to ask

Can we improve the theory for gradient perturbation that can reflect its empirical advantage?

In the analysis for gradient perturbation approach, all previous works  \cite{bassily2014differentially,wang2017differentially,jayaraman2018distributed} derive the utility guarantee via two steps. They first determine the noise variance that meets the privacy requirement and then derive the utility guarantee by using the convergence analysis the same as in the non-private case.  However, the noise to guarantee privacy  naturally affects the optimization procedure, which was neglected in previous analysis. 

In this paper, we utilize the interaction between the privacy noise and the optimization procedure, which help us establish new utility guarantees for gradient perturbation approaches. Our contributions are  as follows.

\begin{itemize}
\item We introduce an \emph{expected curvature} that characterizes the optimization property accurately when there is perturbation noise at each gradient update.

\item We establish new utility guarantees for DP-GD and DP-SGD for both convex and strongly convex objectives based on the \emph{expected curvature} rather than the usual minimum curvature, with significant improvement over previous work.

\item We conduct extensive experiments on real world datasets. Our experiments suggest that gradient perturbation with the advanced composition theorem outperforms other perturbation approaches, which corroborates our theoretical findings nicely.

\end{itemize}

  To the best of our knowledge, this is the first utility guarantee of DP-ERM that removes the dependency on the minimum curvature. Moreover, we argue that objective/output perturbation cannot utilize this expected curvature condition to improve their utility guarantee as no noise is injected in their training process.

{\small
\begin{table}

\centering
\begin{tabular}{|c|c|c|}
\hline
Perturbation 													& Convex & 	$\mu$-S.C.							      \\		

\hline
\begin{tabular}{@{}c@{}}Objective \\ \small{\cite{chaudhuri2011differentially}}\end{tabular}							& $\frac{\sqrt{p}}{n\epsilon}$  & 		$\frac{p}{\mu n^{2}\epsilon^{2}}$		\\	

\hline
\begin{tabular}{@{}c@{}}Output \\  \small{\cite{zhang2017efficient}} \end{tabular}
            															& $(\frac{\sqrt{\beta p}}{n\epsilon})^{2/3}$      & $\frac{\beta p}{\mu^{2} n^{2}\epsilon^{2}}$							   \\

\hline
\begin{tabular}{@{}c@{}} Gradient (SGD) \\  \small{\cite{bassily2014differentially}} \end{tabular}			& $\frac{\sqrt{p}\log^{3/2}(n)}{n\epsilon}$  &	$\frac{p\log^{2}\left(n\right)}{\mu n^{2}\epsilon^{2}}$				\\

\hline
\begin{tabular}{@{}c@{}} Gradient (GD) \\  \small{\cite{jayaraman2018distributed}} \end{tabular}
 		&  N/A   & 		$	\frac{\beta p \log^{2}(n)}{\mu^{2} n^{2}\epsilon^{2}}$						 			\\

\hline
DP-GD (Ours)   	   																									& $\frac{\sqrt{p}}{n\epsilon}\wedge \frac{\beta p \log(n)}{ \nu^{2}n^{2}\epsilon^{2}}$   & $\leftarrow$   \\
\hline	
DP-SGD (Ours)																										& $\frac{\sqrt{p}\log(n)}{n\epsilon} \wedge \frac{ p \log(n)}{ \nu n^{2}\epsilon^{2}}$  &  $\leftarrow$	\\
\hline	
\end{tabular}
\caption{Expected excess empirical risk bounds (the smaller, the better) under $(\epsilon, \delta)$-DP. Symbol $\leftarrow$ denotes the bound is same as the one on the left. Notations: $n$ and $p$ are  the number of samples and the number of parameters,  respectively, and $\beta, \mu$ and $\nu$ are the smooth coefficient, the strongly convex coefficient and the  \emph{expected curvature}, respectively,  and  $\nu\geq\mu$ (see Section~\ref{sec:main}). All bounds should be read as $\mathcal{O}(\cdot)$. Operator $\wedge$ selects the smaller element.  The Lipschitz constant $L$ is assumed to be $1$. We omit   $\log\left(1/\delta\right)$ for simplicity.}
\label{tb:utility}
\end{table}
}

\section{Preliminary}
\label{sec:pre}

We introduce notations and definitions in this section.  We use $D=\{d_{1},\ldots,d_{n}\}$ to denote  a dataset with $n$ records. Each $d_{i}=\{\vs_{i}, t_{i}\}$ contains a feature vector $\vs_{i}$  and a target $t_{i}$.  The objective function $F(\vx;D)$ is defined as $F(\vx;D)\delequal\frac{1}{n}\sum_{i=1}^{n}f(\vx;d_{i})$, where $f(\vx;d_{i}):\mathbb{R}^{p}\rightarrow \mathbb{R}$ is the loss of model $\vx\in\mathbb{R}^{p}$ for the record $d_{i}$.

For simplicity, we use $F(\vx)$ to denote $F(\vx;D)$ and $f_{i}(\vx)$ to denote $f(\vx; d_{i})$. We use $\norm{\vv}$ to denote the $l_{2}$ norm of a vector $\vv$. We use $\mathcal{X}_{f}^{*}=\argmin_{\vx\in\mathbb{R}^{p}}f(\vx)$ to denote the set of optimal solutions of $f(\vx)$. Throughout this paper, we assume $\mathcal{X}_{f}^{*}$  non-empty. % and $\vx_{prj}$ is attainable. Definition~\ref{def:properties} shows some assumptions on $F(\vx)$.

\begin{definition}[Objective properties]%[$L$-lipschitz,  $\mu$-strongly convex and $\beta$-smooth]
\label{def:properties}
For any $\vx, \vy\in\mathbb{R}^{p}$ , a function $f:\mathbb{R}^{p}\rightarrow \mathbb{R}$
\begin{itemize}
	\item is $L$-Lipschitz if $|f(\vx)-f(\vy)|\leq L\norm{\vx-\vy}$.
	\item is $\beta$-smooth if $\idot{\nabla f(\vx)-\nabla f(\vy), \vx-\vy}\leq\beta\norm{\vx-\vy}^{2}$.
	\item is  convex if $\idot{\nabla f(\vx)-\nabla f(\vy), \vx-\vy}\geq 0$.
	\item is $\mu$-strongly convex (or $\mu$-S.C.) if $\idot{\nabla f(\vx)-\nabla f(\vy), \vx-\vy}\geq\mu\norm{\vx-\vy}^{2}$.
\end{itemize}
\end{definition}

The strong convexity coefficient $\mu$ is the lower bound of the minimum curvature of function $f$ over the domain. %General 

We say that two datasets $D,D'$ are neighboring datasets (denoted as $D\sim D^{'}$) if $D$ can be obtained by arbitrarily modifying one record in $D'$ (or vice versa).
In this paper we consider the $(\epsilon,\delta)$-differential privacy.
\begin{definition}[$(\epsilon,\delta)$-DP  \cite{dwork2006our}]
A randomized mechanism $\mathcal{M}:D\rightarrow \mathcal{R}$  guarantees $(\epsilon,\delta)$-differential privacy if for any two neighboring input datasets $D,D^{'}$ and for any subset of outputs $S\subseteq \mathcal{R}$ it holds that $\text{Pr}[\mathcal{M}(D)\in S]\leq e^{\epsilon}\text{Pr}[\mathcal{M}(D^{'})\in S]+\delta$.
\end{definition}
The parameter $\delta$ can be viewed as the probability that original $\epsilon$-DP \cite{dwork2006our} fails and a meaningful setting requires $\delta\ll \frac{1}{n}$. By its definition, differential privacy controls the maximum influence that any individual record can produce. Smaller $\epsilon, \delta$ implies less information leak but usually leads to worse utility. One can adjust $\epsilon,\delta$ to trade off between privacy and utility.

DP-ERM requires the output $\vx_{out}\in \mathbb{R}^{p}$ is differentially private with respect to the input dataset $D$.  Let $\vx_{*}\in\mathcal{X}_{F}^{*}$ be one of the optimal solutions of $F(\vx)$, the utility of DP-ERM algorithm is measured by  \emph{expected excess empirical risk}:  $\mathbb{E}[F(\vx_{out})-F(\vx_{*})]$, where the expectation is taken over the  algorithm randomness. Following previous work, we use utility guarantee to denote the upper bound on expected risk. The algorithm with smaller utility guarantee is better.

%% file: mainresult.tex
\section{Main Results}
\vspace{-1mm}
\label{sec:mainmain}

In this section,  we first define the \emph{expected curvature} and discuss its property. 
We then use such expected curvature to improve the utility analysis for  DP-SGD and DP-GD.

\subsection{Expected Curvature}
\label{sec:main}

{
\begin{table*}
\centering
\begin{tabular}{ccccc}
\hline
Dataset&Graduate Admission&Sonar&Boston Housing&Swedish Motor\\
\hline
Minimum &$6.68\times 10^{-3}$&$3.63\times 10^{-4}$&$3.44\times 10^{-3}$&$9.07\times 10^{-6}$\\
\hline
Average &0.40&0.19&0.34&0.24\\
\hline
\end{tabular}
    \caption{ Minimum and average curvatures for linear regression with four datasets. The minimum and the average curvatures are the smallest and the average eigenvalues of the Hessian matrix, respectively. We normalize the features with standard min-max normalization for all datasets. }
    \label{tbl:lr-curvatures}
\end{table*}
}

In the non-private setting, the convergence analysis of gradient descent for strongly convex objective explicitly relies on the strongly convex coefficient $\mu$, i.e., the minimum curvature of the objective function over the domain, which can be extremely small for some objectives. Previous work on DP-ERM uses the same analysis as in the non-private case and therefore the resulting utility bounds rely on the minimum curvature. In our analysis, however, we study the interaction between privacy noise and the optimization procedure, and figure out that the optimization procedure is determined by certain \emph{expected curvature} rather than the usual minimum curvature. Definition~\ref{main} describes such \emph{expected curvature} when the perturbation noise is Gaussian. We use $\vx_{*}=\argmin_{\vx\in\mathcal{X}_{*}}{\norm{\vx-\vx_{1}}}$ to denote the closest solution to the initial point.

\begin{definition}[Expected curvature] 
\label{main}
A convex function $F: \sR^p \rightarrow \sR$, has \emph{expected curvature}   $\nu$ with respect to noise $\mathcal{N}(0,\sigma^{2}\mI_{p})$ if for any   $\vx\in \mathbb{R}^{p}$ and $\tilde{\vx}=\vx-\vz$ where $\vz\sim\mathcal{N}(0,\sigma^{2}\mI_{p})$, it holds that
\begin{equation}
\begin{aligned}
\label{eq:main-0}
\E[\idot{\nabla F(\tilde{\vx}), \tilde{\vx}-\vx_{*}}]\geq \nu\mathbb{E}[\norm{\tilde{\vx}-\vx_{*}}^{2}],
\end{aligned}
\end{equation}
where the expectation is taken with respect to $\vz$.
\end{definition}

\begin{remark}
If $F(\vx)$ is $\mu$-strongly convex, then $F(\vx)$ has $\nu$ expected curvature with $\nu\ge \mu$.
\end{remark}

It can be verified that $\nu=\mu$ always holds because of the strongly convex definition. 

In fact, $\nu$ is usually much larger than $\mu$.
Let $\mH_{\vx}=\nabla^{2}F(\vx)$ be the Hessian matrix evaluated at $\vx$. We use Taylor expansion  to approximate the left hand side of Eq~(\ref{eq:main-0}) as follows
\begin{flalign}
\E[\idot{\nabla F(\tilde{\vx}), \tilde{\vx}-\vx_{*}}] &\approx \mathbb{E}[\idot{\nabla F(\vx)-\mH_{\vx}\vz, \vx-\vz-\vx_{*}}] \label{eq:nuleft}\\
& = \idot{\nabla F(\vx), \vx-\vx_{*}} + \mathbb{E}[\vz^TH_{\vx}\vz] \nn \\
& = \idot{\nabla F(\vx), \vx-\vx_{*}} +\sigma^{2}\tr(\mH_{\vx}).\nn
\end{flalign}
The approximation is reasonably accurate for smooth convex objectives. For convex objective, the Hessian matrix is  positive semi-definite and $\tr(\mH_{\vx})$ is the sum of the eigenvalues of $H_{\vx}$. The right hand side of Eq~(\ref{eq:main-0}) can be written as
\begin{flalign}
\nu\mathbb{E}[\norm{\tilde{\vx}-\vx_{*}}^{2}]&=\nu\mathbb{E}[\norm{\vx-\vz-\vx_{*}}^{2}]\\&=\nu\left(\norm{\vx-\vx_{*}}^{2} + p\sigma^{2}\right).
\end{flalign}
Based on the above approximation, we can estimate the value of $\nu$ in Definition~\ref{main}: 
\begin{equation}
\begin{aligned}
\label{eq:eval_nu}
\nu\approx\frac{\tr(\mH_{\vx})\sigma^{2}+\idot{\nabla F(\vx), \vx-\vx_{*}}}{p\sigma^{2}+\norm{\vx-\vx_{*}}^{2}}.
\end{aligned}
\end{equation}
 We note that $\idot{\nabla F(\vx), \vx-\vx_{*}}\ge \mu \norm{\vx-\vx_{*}}^{2}$ because of the definition of strongly convex definition\footnote{In fact, this is the restricted strongly convex with respect to the minimizer $\vx_*$.}. For bounded $\norm{\vx-\vx_{*}}$ and relatively large $\sigma^{2}$, this implies $\nu\approx\frac{\tr(\mH_{\vx})}{p}$, the average curvature at $\vx$.  Large variance is a reasonable setting because meaningful differential privacy guarantee requires non-trivial amount of noise.

The above analysis suggests that $\nu$ can be independent of and much larger than $\mu$.  We next show that for linear regression, $\nu$ is  larger than $\mu$ without approximation.
\begin{example}
\label{clm:clm1}
Suppose $F(\vx)=\frac{1}{2n}\sum_{i}\norm{\vx\vs_{i}-t_{i}}^{2}$ is the linear regression objective and $\vx\in\mathcal{C}$, where $\mathcal{C}$ is bounded convex set. Let $\beta, \mu$ be the smooth and the strongly convex coefficients, respectively, and $\beta>\mu$.  Then the $\nu>\mu$ for any $\sigma>0$. 
\end{example}

\begin{proof} 
Let $\mS=[\vs_{1},\ldots,\vs_{n}]^T\in \mathbb{R}^{n\times p}$ be the feature matrix. For linear regression, the Hessian matrix of $F(\vx)$ for any $\vx$ is $\mH=\nabla^{2}F(\vx)=(1/n)\mS^{T}\mS$. The approximation in \eqref{eq:nuleft} becomes equality because higher order derivatives are $0$ for linear regression objective. 

Thus, we have Eq~(\ref{eq:main-0}) hold for $\nu=\frac{\tr(\mH)\sigma^{2}+\idot{\nabla F(\vx), \vx-\vx_{*}}}{p\sigma^{2}+\norm{\vx-\vx_{*}}^{2}}$. Because $\mu$ is the smallest eigenvalue of $\mH$, we have $\tr(\mH)>p\mu$ and therefore $\nu>\mu$.
\end{proof}

In practice, the Hessian matrix of linear regression could be ill-conditioned (the minimum curvature is extremely small). Table~\ref{tbl:lr-curvatures} shows the minimum and the average curvatures of four real world regression datasets collected from the Kaggle website\footnote{\url{https://www.kaggle.com/}}.  We further plot the expected curvatures with varying $\sigma$ in Figure~\ref{fig:curvatures_varying_sigma}. We use $\vx=0^{p}$ and Eq~(\ref{eq:eval_nu}) to compute the expected curvature. The expected curvature is almost as large as the average curvature when $\sigma=1.0$.

{
\begin{figure}
\begin{center}
  \includegraphics[width=.6\linewidth]{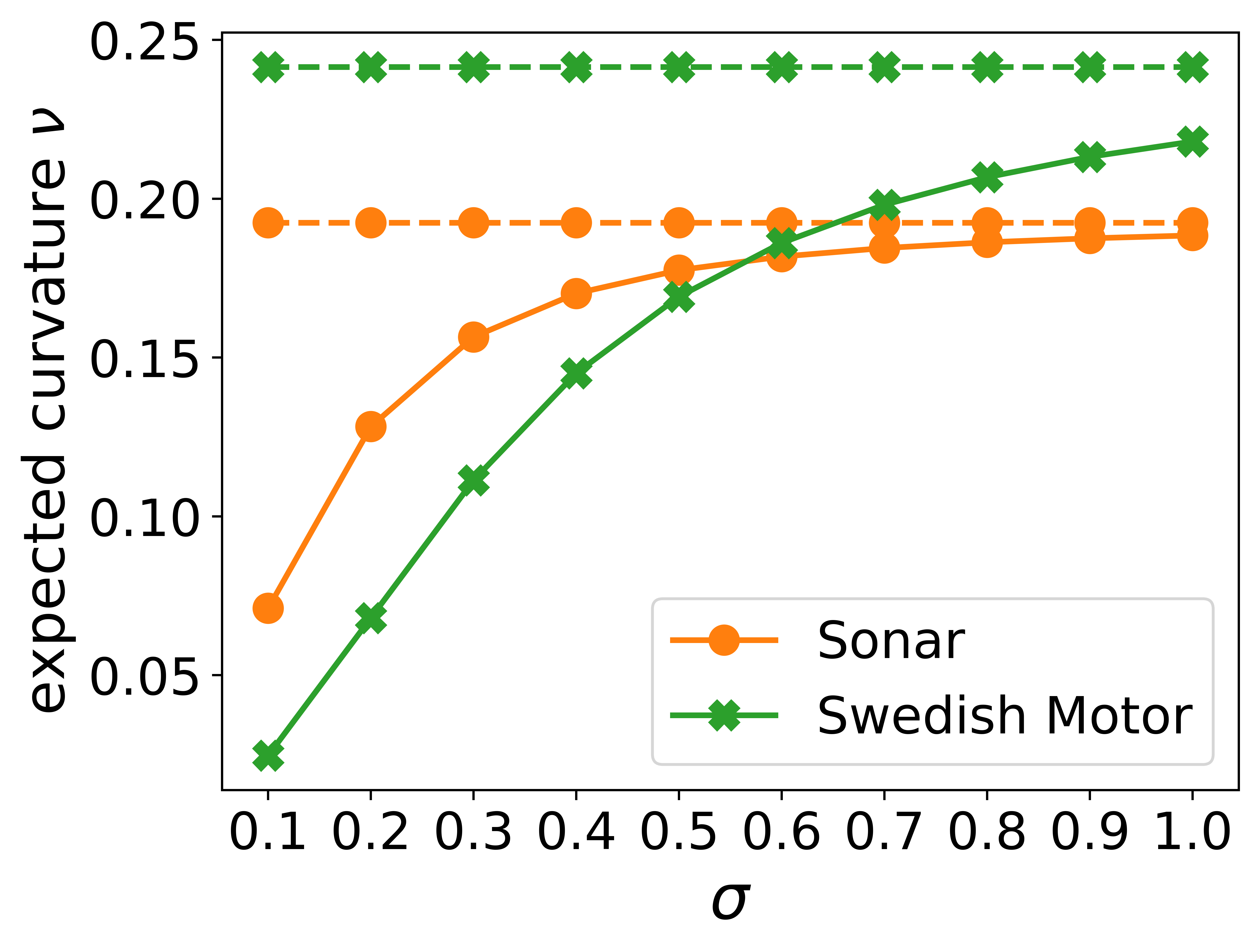}
  \caption{ Expected curvatures $\nu$ with varying $\sigma$ of two regression datasets. The dash line is the average curvature. The solid line is the expected curvature $\nu$. We compute $\nu$ using Eq~(\ref{eq:eval_nu}) with $\vx=0$.}
  \label{fig:curvatures_varying_sigma}
 \end{center}
\end{figure}
}

Another example is the $l_2$ regularized logistic regression. The objective is strongly convex only due to the $l_2$ regularizer. Thus, the minimum curvature (strongly convex coefficient) is the regularization coefficient $\lambda$. \cite{shamir2014communication} shows an optimal choice of $\lambda$ is $\Theta(n^{-1/2})$. In practice, typical choice of $\lambda$ is even smaller and could be on the order of $n^{-1}$.   Therefore removing the dependence on minimum curvature is a significant improvement. 

%Figure~\ref{fig:nu} compares the minimum and average curvatures of regularized logistic regression during the training process. The average curvature is basically unaffected by the regularization  term $\lambda$. In contrast, the minimum curvature reaches $\lambda$ in first few steps.

Another important class of objectives are the restricted strongly convex functions, which naturally arise in the under-determined systems. Although $\mu=0$ for such case, $\nu$ is strictly positive and is at least as large as the restricted strongly convex coefficient. Our theory can easily handle the restricted strongly convex function class.%  \textbf{We note that $\mu=0$ does not 

% \begin{figure}
% \centering
%   \includegraphics[width=0.6\linewidth]{imgs/curvatures.png}
%   \caption{\small Curvatures of regularized logistic regression on Adult dataset over training. See Section~\ref{sec:exp} for more info of Adult dataset. Dot/cross symbol represents average/minimum curvature respectively.}
%   \label{fig:nu}
% \end{figure}

\subsection{New Utility Guarantee for DP-GD}% Based on Expected Curvature}
\label{sec:main-2}
%\vspace{-1mm}

Algorithm~\ref{alg:DP-GD} describes the DP-GD algorithm (gradient perturbation). It has been shown that for the chosen $\sigma$, Algorithm~\ref{alg:DP-GD} is $(\epsilon, \delta)$-DP \cite{jayaraman2018distributed}.   We next present a new improved utility guarantee of DP-GD by using the \emph{expected curvature}.

%In this section we improve the utility bound of DP-GD (Algorithm~\ref{alg:DP-GD}) by using the \emph{expected curvature}. 
%For smooth objective, GD has linear convergence rate therefore needs much fewer iterations to achieve the same utility as SGD. %Tough the computation cost of each iteration becomes higher, it now gives more possibility for parallelism. We do not cover non-smooth setting for DP-GD because full gradient has no clear advantage compared to stochastic gradient if without smooth assumption. 

\begin{algorithm}[tb]
\caption{DP-GD}
   \label{alg:DP-GD}
\begin{algorithmic}[1]
   \STATE {\bfseries Input:} running steps $T$; learning rate $\eta$. 
   %Loss function $F(\vx)$ with Lipschitz constant $L$.  
   \STATE Set $\sigma=\frac{2\sqrt{2}L\sqrt{T\log(1/\delta)}}{n\epsilon}$.
   \FOR{$t=1$ {\bfseries to} $T$}
			\STATE Compute $\vg_{t} = \nabla F\left(\vx_{t}\right)$.
			%\STATE Set $\sigma=\frac{2\sqrt{2}L\sqrt{T\log\left(1/\delta\right)}}{n\epsilon}$.
		   \STATE Update parameter $\vx_{t+1}=\vx_{t}-\eta\left(\vg_{t}+\vz_{t}\right)$, where $\vz_{t}\sim\mathcal{N}\left(0,\sigma^{2}I_{p}\right)$.
   \ENDFOR
\end{algorithmic}
\end{algorithm}

\begin{theorem} 
\label{thm:DP-GD}
Suppose $F$ is $L$-Lipschitz and $\beta$-smooth with $\nu$ expected curvature. 

If $\nu>0$,  letting $\eta\leq\frac{1}{\beta}$ and $T=\frac{2\log\left(n\right)}{\eta\nu}$, then we have
\[\mathbb{E}[F\left(\vx_{T+1}\right)-F\left(\vx_{*}\right)]=\mathcal{O}\left(\frac{\beta p\log\left(n\right) L^{2}\log\left(1/\delta\right)}{\nu^{2} n^{2}\epsilon^{2}}\right).\]

If $\nu=0$, letting $\eta=\frac{1}{\beta}$,  $T=\frac{n\beta\epsilon}{\sqrt{p}}$ and $\bar \vx=\frac{1}{T}\sum_{i=1}^{T}\vx_{i+1}$, then we have
\[\mathbb{E}[F\left(\bar \vx\right)-F\left(\vx_{*}\right)]=\mathcal{O}\left(\frac{\sqrt{p}L^{2}\log\left(1/\delta\right)}{n\epsilon}\right).\]
The expectations are taken over the algorithm randomness.
\end{theorem}
\begin{proof}
We first prove the case of $\nu>0$. Since $\vx_{t}$ contains Gaussian perturbation noise $\vz_{t-1}$, Definition~\ref{main} gives us
\[\mathbb{E}_{\vz_{t-1}}[\idot{\vx_{t}-\vx_{*}, \nabla F\left(\vx_{t}\right)}]\geq\nu\mathbb{E}_{\vz_{t-1}}[\norm{\vx_{t}-\vx_{*}}^{2}].\]

By the $\beta$-smooth condition, we have
\begin{equation}
\begin{aligned}
\label{eq:rt-0}
&\mathbb{E}_{\vz_{t-1}}[\idot{\vx_{t}-\vx_{*}, \nabla F\left(\vx_{t}\right)}] \\
&\geq \frac{\nu}{2}\mathbb{E}_{\vz_{t-1}}[\norm{\vx_{t}-\vx_{*}}^{2}] + \frac{1}{2\beta}\mathbb{E}_{\vz_{t-1}}[\norm{\nabla F\left(\vx_{t}\right)}^{2}].
\end{aligned}
\end{equation}
Let $r_{t}=\norm{\vx_{t}-\vx_{*}}$ be the estimation error at step $t$. The update rule gives us
\begin{equation}
\begin{aligned}
\label{eq:rt-1}
&r_{t+1}^{2}=\norm{\vx_{t}-\eta \nabla F\left(\vx_{t}\right)-\eta \vz_{t}-\vx_{*} }^{2}.
\end{aligned}
\end{equation}
Taking expectation of  Eq~(\ref{eq:rt-1}) with respect to $\vz_{t}, \vz_{t-1}$ and using Eq~(\ref{eq:rt-0}), we have
\begin{equation}
\begin{aligned}
\label{eq:rt-rsc-smooth}
&\mathbb{E}_{\vz_{t},\vz_{t-1}}[r_{t+1}^{2}]\leq\left(1-\eta \nu\right)\mathbb{E}_{\vz_{t-1}}[r_{t}^{2}]\\
& \;+\left(\eta^{2}-\frac{\eta}{\beta}\right)\mathbb{E}_{\vz_{t-1}}[\norm{\nabla F\left(\vx_{t}\right)}^{2}] + p\eta^{2}\sigma
^{2}.
\end{aligned}
\end{equation}
Choose  $\eta\leq\frac{1}{\beta}$.  Applying Eq~(\ref{eq:rt-rsc-smooth})  and taking expectation with respect to $\vz_t, \vz_{t-1}, \cdots, \vz_1$ iteratively  yields
\begin{equation}
\begin{aligned}
\label{eq:rt-5}
&\mathbb{E}[r_{t+1}^{2}]\leq \left(1-\eta\nu\right)^{t}r_{1}^{2}+ p\eta^{2}\sum_{i=1}^{t}\left(1-\eta\nu\right)^{t-i}\sigma_{i}^{2}.
\end{aligned}
\end{equation}
Then plug  $t=2\log(n)/\eta\nu$, $\sigma=L\sqrt{T\log(1/\delta)/n\epsilon}$  in Eq~(\ref{eq:rt-5}).  We obtain the desired bound by using the smoothness condition to convert Eq~(\ref{eq:rt-5}) to the objective function value bound. 

%Replace $\vy, \vx$ with $\vx_{t+1}, \vx_{t}$, 

We now prove the case of $\nu=0$. The  $\beta$-smooth condition gives us
\[F(\vx_{t+1})\leq F(\vx_{t})+\idot{\nabla F(\vx_{t}), \vx_{t+1}-\vx_{t}}+\frac{\beta}{2}\norm{\vx_{t+1}-\vx_{t}}^{2}.\]
%Substitute $\vx_{t+1}-\vx_{t}=-\eta(\nabla F(\vx_{t})+\vz_{t})$, s
Set $\eta=\frac{1}{\beta}$ and take expectation with respect to $\vz_{t}$.
%\begin{equation}
%\begin{aligned}
%\label{eq:dpgdc-3}
\[\mathbb{E}_{\vz_{t}}[F(\vx_{t+1})] = F(\vx_{t}) - \frac{1}{2}\norm{\nabla F(\vx_{t})}^{2}/\beta  + \frac{1}{2}p\sigma^{2}/\beta.\]
%\end{aligned}
%\end{equation}
%&\leq \idot{\nabla F(\vx_{t}), \vx_{t}-\vx_{*}}  - \frac{1}{2\beta}\norm{\nabla F(\vx_{t})}^{2}  + \frac{1}{2\beta}p\sigma^{2}.
Subtract $F(\vx_{*})$ from both sides and use convexity.
\begin{equation}
\begin{aligned}
\label{eq:dpgdc-4}
&\mathbb{E}_{\vz_{t}}[F(\vx_{t+1})-F(\vx_{*})]\\&\leq \idot{\nabla F(\vx_{t}), \vx_{t}-\vx_{*}}  - \frac{1}{2\beta}\norm{\nabla F(\vx_{t})}^{2}  + \frac{1}{2\beta}p\sigma^{2}.
\end{aligned}
\end{equation}
Plug in $\nabla F(\vx_{t})=\beta(\vx_{t}-\vx_{t+1})-\vz_{t}$ and rearrange. 
\begin{equation}
\begin{aligned}
\label{eq:dpgdc-5}
&\mathbb{E}_{\vz_{t}}[F(\vx_{t+1}) - F(\vx_{*})] \\
&\leq \frac{\beta}{2}(\norm{\vx_{t}-\vx_{*}}^{2} - \norm{\vx_{t+1}-\vx_{*}}^{2}) + \frac{1}{\beta}p\sigma^{2}.
\end{aligned}
\end{equation}
Sum over $t=1,\ldots,T$,  take expectation with respect to $\vz_{1},\ldots,\vz_{T}$ and use the convexity. 
\begin{equation}
\begin{aligned}
\label{eq:dpgdc-6}
\mathbb{E}[F(\bar \vx) - F(\vx_{*})] &\leq \frac{\beta}{2T}\norm{\vx_{1}-\vx_{*}}^{2} +  \frac{1}{\beta}p\sigma^{2}.
\end{aligned}
\end{equation}
Plugging in $T$ and $\sigma$ yields the desired bound.
\end{proof}

\begin{remark}

Theorem~\ref{thm:DP-GD} only  depends  on the expected curvature $\nu$ over the training path.
\end{remark}
 Theorem~\ref{thm:DP-GD}  significantly improves the original analysis of DP-GD because of our argument in Section~\ref{sec:main}. We note that if $\nu=0$, then the curvatures are flatten in all directions. One example is the dot product function, which is used by \cite{bassily2014differentially} to derive their utility lower bound. Such simple function may not be commonly used as loss function in practice. 
\subsection{New Utility Guarantee for DP-SGD}% Based on Expected Curvature}
\label{sec:DP-SGD}
%\vspace{-1mm}

Stochastic gradient descent has become one of the most popular optimization methods because of the cheap one-iteration cost. In this section we describe the DP-SGD (Algorithm~\ref{alg:DP-SGD}) and show that \emph{expected curvature} can also improve the utility analysis for DP-SGD. We note that $\nabla f(\vx)$ represents an element from the subgradient set evaluated at $\vx$ when the objective is not smooth. Before stating our theorems, we introduce the \emph{moments accountant} technique (Lemma~\ref{lem:privacy DP-SGD}) that is essential to establish  privacy guarantee.

\begin{lemma}[\cite{abadi2016deep}]
\label{lem:privacy DP-SGD}
There exist constants $c_1$ and $c_2$ so that given running steps $T$, for any $\epsilon<c_{1}T/n^2$, Algorithm~\ref{alg:DP-SGD} is $\left(\epsilon,\delta\right)$-differentially private for any $\delta>0$ if we choose $\sigma\geq c_2\frac{\sqrt{Tlog\left(1/\delta\right)}}{n\epsilon}$.
\end{lemma}

\begin{algorithm}[tb]
	\caption{DP-SGD}
	\label{alg:DP-SGD}
	%\SetKwInOut{Input}{Input}
	%\SetKwInOut{Output}{Output}
	%\textbf{Input}:{}
\begin{algorithmic}[1]
	
	%Dataset $D=\{d_{1},\ldots,d_{n}\}$. Individual loss function: $f_{i}\left(\vx\right)=f\left(\vx;d_{i}\right)$ with Lipschitz constant $L$.
	\STATE {\bfseries Input:} running steps $T$; learning rate $\eta_t$.
	\FOR{$t=1$ {\bfseries to} $T$} 
		\STATE Sample $i_{t}$ from $[n]$ uniformly.

		\STATE Compute $\vg_t=\nabla f_{i_{t}}\left(\vx_{t}\right)$.

		\STATE Update parameter $\vx_{t+1}=\vx_{t}-\eta_{t}\left(\vg_{t}+\vz_{t}\right)$, where $\vz_{t}\sim\mathcal{N}\left(0,L^{2}\sigma^{2}I_{p}\right)$.
    \ENDFOR
	
\end{algorithmic}
\end{algorithm}

We next establish the utility guarantee of Algorithm~\ref{alg:DP-SGD} in Theorem~\ref{thm:DP-SGD}.

\begin{theorem}
\label{thm:DP-SGD} 
Suppose $F$ is $L$-Lipschitz with $\nu$ expected curvature. Choose $\sigma$ as stated in Lemma~\ref{lem:privacy DP-SGD}. 

If $\nu>0$, letting $\eta_{t}=\frac{1}{\nu t}$ and $T= n^{2}\epsilon^{2}$, then we have 
\[\mathbb{E}[F\left(\vx_{T}\right)-F\left(\vx_{*}\right)]=\mathcal{O}\left(\frac{pL^{2}\log\left(n\right)\log\left(1/\delta\right)}{ n^{2}\epsilon^{2}\nu}\right).\]

If $\nu=0$, letting $G=L\sqrt{1+p\sigma^{2}}$,  $\eta_{t}=\frac{D}{G\sqrt{t}}$ and $T=n^{2}\epsilon^{2}$, then we have 
\[\mathbb{E}[F\left(\vx_{T}\right)-F\left(\vx_{*}\right)]=\mathcal{O}\left(\frac{\sqrt{p\log\left(1/\delta\right)}L\log\left(n\right)}{n \epsilon}\right).\]
\end{theorem}

\begin{proof}
For the case of $\nu>0$, we first show that $r_{t}:=\norm{\vx_{t}-\vx_{*}}$  satisfies 
\begin{equation}
\begin{aligned}
\label{eq:lma1-1}
\mathbb{E}[r_{t}^{2}]\leq\frac{2L^{2}\left(1+p\sigma^{2}\right)}{t\nu^{2}},
\end{aligned}
\end{equation}
which depends on $\nu$ rather than  $\mu$. Using the update rule to express out $r_{t+1}$ and taking expectation with respect to $\vz_{t}$ and $i_{t}$, we have
{\begin{equation}
\begin{aligned}
\label{eq:lma1-2}
\mathbb{E}_{\vz_{t},i_{t}}[r_{t+1}^{2}]&\leq r_{t}^{2}-2\eta_{t}\idot{\vx_{t}-\vx_{*},\nabla F\left(\vx_{t}\right)}\\
&+\eta_{t}^{2}L^{2}+p\eta_{t}^{2}L^{2}\sigma^{2}.
\end{aligned}
\end{equation}}
Taking expectation to $\vz_{t-1}$ and using Definition~\ref{main}, we have
\begin{equation}
\begin{aligned}
\label{eq:lma1-3}
\mathbb{E}[r_{t+1}^{2}]&\leq\left(1-2\nu\eta_{t}\right)\mathbb{E}[r_{t}^{2}]+\eta_{t}^{2}L^{2}\left(1+p\sigma^{2}\right).
\end{aligned}
\end{equation}
Now we use induction to conduct the proof. Plugging $\eta_{t}=\frac{1}{t\nu}$ into Eq~(\ref{eq:lma1-3}), we have Eq~(\ref{eq:lma1-1}) hold for $t=2$. Assuming that $\mathbb{E}[r_{t}^{2}]\leq\frac{2L^{2}\left(1+p\sigma^{2}\right)}{t\nu^{2}}$ holds for $t>2$, then 
{\begin{equation}
\begin{aligned}
\label{eq:lma1-4}
\mathbb{E}[r_{t+1}^{2}]&\leq\left(1-\frac{2}{t}\right)\mathbb{E}[r_{t}^{2}]+\frac{L^{2}\left(1+p\sigma^{2}\right)}{\nu^{2}t^{2}}\\
&\leq\left(\frac{2}{t}-\frac{3}{t^{2}}\right)\frac{L^{2}\left(1+p\sigma^{2}\right)}{\nu^{2}} \leq \frac{2L^{2}\left(1+p\sigma^{2}\right)}{\left(t+1\right)\nu^{2}}.
\end{aligned}
\end{equation}}
We next use the results in \cite{shamir2013stochastic} to derive the final utility bound. \citeauthor{shamir2013stochastic} show how to  convert the convergence in $r_{t}$  into the convergence  in objective function value without any smooth assumption. The utility guarantee for the case of $\nu>0$ can be derived by plugging Eq~(\ref{eq:lma1-1}) and $T=n^{2}\epsilon^{2}$ into the Theorem 1 of \cite{shamir2013stochastic}. 

For the case of $\nu=0$, the proof is similar to the $\mu=0$ case in  \cite{bassily2014differentially}. We improve the bound by a $\log(n)$ factor because we use the advanced \emph{moments accountant} technique.

\end{proof}

% \begin{remark}
% Theorem~\ref{thm:DP-SGD} does not require smooth assumption.
% \end{remark}

Theorem~\ref{thm:DP-SGD} shows the utility guarantee of DP-SGD depends on $\nu$ rather than $\mu$. We set  $T=\Theta(n^{2})$ following \cite{bassily2014differentially}. We note that $T=\Theta(n^{2})$ is necessary even for non-private SGD to reach  $1/n^{2}$ precision. If we choose $T=\frac{n\epsilon}{\sqrt{p}}$ in the proof of Theorem~\ref{thm:DP-SGD}, then the utility bound for $\nu>0$ is
\[\mathbb{E}[F\left(\vx_{T}\right)-F\left(\vx_{*}\right)]=\mathcal{O}\left(\frac{\sqrt{p}L^{2}\log(n)}{n\epsilon\nu}\right).\]

In contrast, the analysis of \cite{bassily2014differentially} yields 
$\mathbb{E}[F\left(\vx_{T}\right)-F\left(\vx_{*}\right)]=\mathcal{O}\left(\frac{\sqrt{p}L^{2}\log^{2}(n)}{n\epsilon\mu}\right)$ if  setting $T= \frac{n\epsilon}{\sqrt{p}}$, which still depends on the minimum curvature.

\subsection{Discussion}

%Take the sensitivity of gradient descent for example, the gradient  of two neighboring datasets only differ in one record at each update.
In this section, we briefly discuss two other perturbation approaches and compare them to the gradient perturbation approach.

\emph{Output perturbation} \cite{wu2017bolt,zhang2017efficient} first runs the learning algorithm the same as in the non-private case and then adds data-independent noise to the output parameter according to sensitivity of the whole learning procedure.  The magnitude of perturbation noise is propositional to the maximum influence  one record can cause on the learned model over the whole learning procedure. The sensitivity is determined by the Lipschitz coefficient (to control the maximum expansion of two different records), and the smooth and the strongly convex coefficients (to measure the contraction effect of gradient update of the ``good'' records). 

\emph{Objective perturbation} \cite{chaudhuri2011differentially,kifer2012private,iyengar2019towards} perturbs the objective (i.e., the empirical loss)  with a random linear term and then releases the minimizer of the perturbed objective. The strong convexity is necessary to ensure the minimizer is unique\footnote{It adds $L_{2}$ regularization to obtain strong convexity if the original objective is not strongly convex.}.   The sensitivity of objective perturbation is the maximum change of the minimizer that one record can produce. \cite{chaudhuri2011differentially}  uses the largest and the smallest eigenvalue (i.e. the smooth and strongly convex coefficient) of the objective's Hessian matrix to upper bound such change.   

In comparison, \emph{Gradient perturbation} is more flexible than the output/objective perturbations. For gradient perturbation, the sensitivity is only determined by the Lipschitz coefficient which is easy to obtain by using the gradient clipping technique. More critically, we always have $\nu=\mu$ when the gradient update does not involve perturbation noise. Therefore, gradient perturbation is the only  method can leverage such effect of noise among three existing perturbation methods.

%% file: numerical.tex
%\vspace{-1mm}
\section{Experiment}
\label{sec:exp}

In this section, we evaluate the performance of DP-GD and DP-SGD on multiple real world datasets.  Objective functions are \emph{logistic regression} and \emph{softmax regression} for binary and multi-class datasets, respectively.

\paragraph{Datasets.} We present the results of four benchmark datasets in \cite{iyengar2019towards}, including one multi-class dataset (MNIST) and two with high dimensional features (Real-sim, RCV1)\footnote{We conduct experiments on all 7 datasets in \cite{iyengar2019towards} and present 4 of them due to limited space.}.   Detailed description of datasets can be found in Table~\ref{tbl:4-datasets}. We use the same pre-processing as in \cite{iyengar2019code}. We use $80\%$ data for training and the rest for testing, the same as \cite{iyengar2019towards}.

\begin{table} 

\centering
\begin{TAB}(r)[0.5pt]{|c|c|c|c|c|}{|c|c|c|c|}

dataset						&\text{Adult}		& MNIST	  & \text{Real-sim} & RCV1	\\		
\# records				    & 45220    		& 65000   						&72309			&50000  				\\	
\# features					&	104				&	784			&20958			&47236				\\
\# classes					&2			&10			&2	&2 \\
\end{TAB}
\caption{Detailed description of four real world datasets.}
\label{tbl:4-datasets}
\end{table}

\begin{figure}
\centering
  \includegraphics[width=0.9\linewidth]{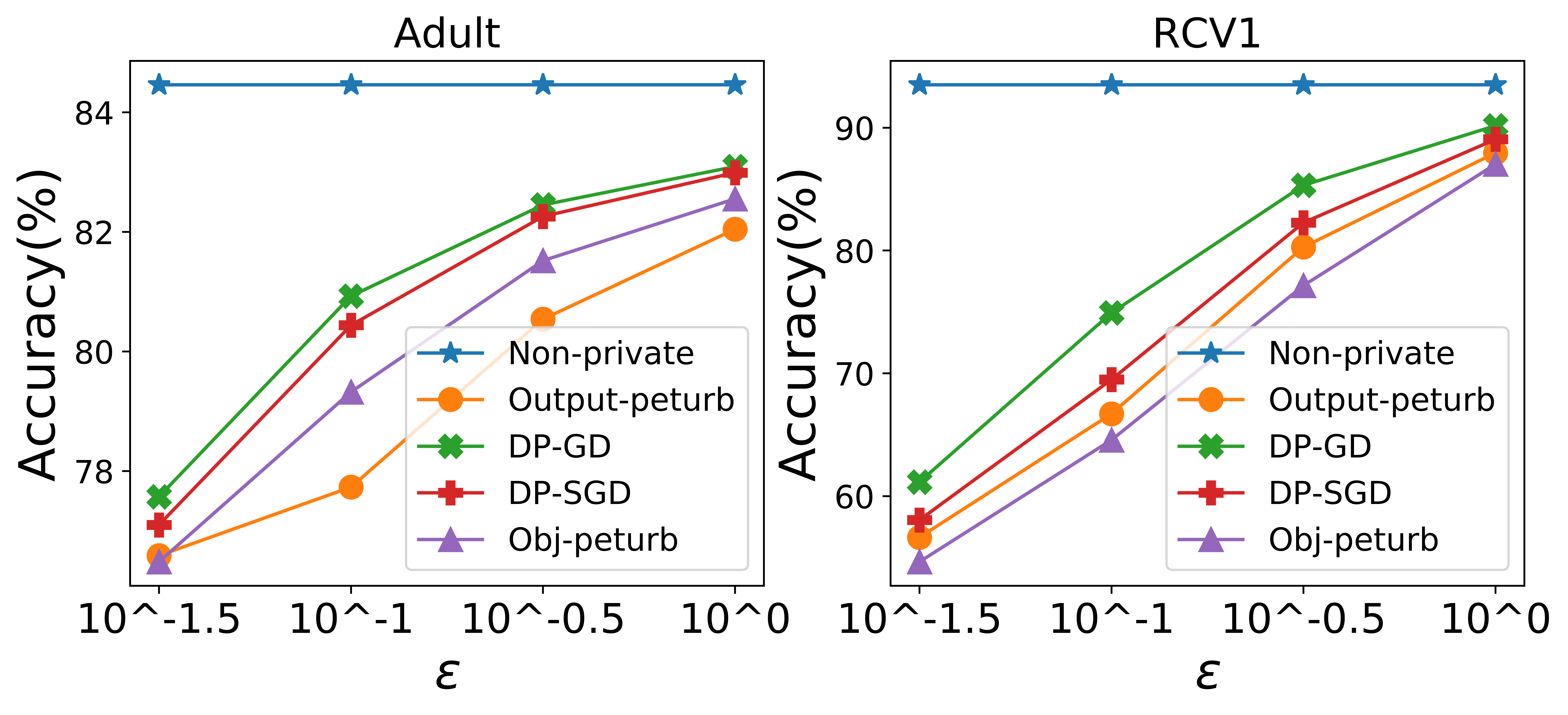}
  \caption{Algorithm validation accuracy (in \%) with varying $\epsilon$. NP represents non-private baseline.}
  \label{fig:res_7_sets}
  \centering
\end{figure}

\paragraph{Implementation details.}  We track \emph{Rényi differentialy privacy (RDP) \cite{mironov2017renyi}} and convert it to $(\epsilon, \delta)$-DP. Running step $T$ is chosen from $\{50, 200, 800\}$ for both DP-GD and DP-SGD. The standard deviation of the added noise $\sigma$ is set to be the smallest value such that the privacy budget is allowable to run desired steps.  We ensure each loss function is Lipschitz by clipping individual gradient. The method in  \cite{goodfellow2015efficient} allows us to clip individual gradient efficiently.  Clipping threshold is set as $1$ ($0.5$ for high dimensional datasets because of the sparse gradient).  Privacy parameter $\delta$ is set as $\frac{1}{n^{2}}$. The $l_{2}$ regularization coefficient is set as $1\times10^{-4}$. For DP-GD, learning rate is chosen from $\{0.1, 1.0, 5.0\}$ ($\{0.2, 2.0, 10.0\}$ for high dimensional datasets). For DP-SGD, we use moments accountant to track the privacy loss and the sampling ratio is set as $0.1$ (roughly the mini-batch size is $0.1\times$dataset size).  The learning rate of DP-SGD is twice as large as DP-GD and it is divided by $2$ at the middle of training.  All reported numbers are averaged over 20 runs.  

\paragraph{Baseline algorithms.}  
The baseline algorithms include state-of-the-art objective and output perturbation algorithms. 
For the objective perturbation, we use \emph{Approximate Minima Perturbation (AMP)} \cite{iyengar2019towards}.  For the output perturbation, we use the algorithm in  \cite{wu2017bolt} (Output perturbation SGD). We adopt the implementation and hyperparameters in  \cite{iyengar2019towards} for both algorithms. For the multi-class classification tasks,  \cite{wu2017bolt} and \cite{iyengar2019towards} divide the privacy budget evenly and train multiple binary classifiers because their algorithms need to compute smooth coefficient before training and therefore are   not directly applicable to softmax regression.

\paragraph{Experiment results.} The validation accuracy results for all evaluated algorithms with $\epsilon=0.1$ ($1.0$ for multi-class datasets) are presented in Table~\ref{tb:accuracy}. We also plot the accuracy results with varying $\epsilon$ in Figure~\ref{fig:res_7_sets}.    These results confirm our theory in Section~\ref{sec:mainmain}: gradient perturbation achieves better performance than other perturbation methods as it leverages the average curvature $\nu$. We note that in \cite{iyengar2019towards} the performance of DP-SGD is not as good as in our experiment.  Our setup has two differences from theirs. First, they set the noise variance to be  the upper bound (Algorithm 2 in their Appendix C) in Lemma \ref{lem:privacy DP-SGD}. In contrast, we use smaller noise by carefully estimating the exact privacy loss through numerical integration  as  in \cite{abadi2016deep}.  Second, we decay the learning rate of DP-SGD during training while \cite{iyengar2019towards} use constant learning rate. Learning rate decay is important for both theoretical analysis and empirical performance of DP-SGD. The above differences lead to a significant improvement on the performance of DP-SGD.

\begin{table}
\begin{tabular}{ p{1.5cm}p{1.15cm}p{1.15cm}p{1.35cm} p{1.15cm} }
 \hline
               & Adult  			& MNIST  		& Real-sim 				& RCV1 \\[0.3ex]
 \hline
Non-private   & 84.8 	     & 	91.9			& 	93.3  				& 93.5	\\[0.3ex]
AMP&  79.3	 		 & 71.9 	& 	73.1	 			& 64.5	\\[0.3ex]
Out-SGD  &  77.4	    &69.4			& 73.2	 						& 66.7	\\[0.3ex] 
DP-SGD	   &  	80.4 		&87.5    	& 	73.8	 				   & 70.4	\\[0.3ex]	
DP-GD	   &  	\textbf{80.9} 	  	& \textbf{88.6}   	  			& 	\textbf{76.1}	 	  & 	\textbf{74.9}\\[0.3ex]	
%DP-GD-D 	&	\textbf{98.8}	& \textbf{81.2}& \textbf{88.7}	 		& 	66.2& 	\textbf{68.8}		& \textbf{76.6}		&\textbf{75.2}		\\[0.3ex]
 \hline
\end{tabular}
    \caption{\small Algorithm validation accuracy (in \%) on various kinds of real world datasets. Privacy parameter $\epsilon$ is $0.1$ for binary dataset and $1$ for  multi-classes datasets. }
\label{tb:accuracy}
\end{table}

\section{Conclusion}
\label{sec:conclusion}
In this paper, we show the privacy noise helps optimization analysis, which is used to improve the utility guarantee of both DP-GD and DP-SGD.   Experiments on real world datasets corroborate our theoretical findings nicely. In the future, it is interesting to consider utilizing the expected curvature to improve the utility guarantee of other gradient perturbation based algorithms.

\section*{Acknowledgements}

The authors want to thank the anonymous reviewers for their constructive comments. This work is supported by the National Natural Science Foundation of China (U1711262, U1611264,U1711261,U1811261,U1811264), National Key R\&D Program of China (2018YFB1004404), Guangdong Basic and Applied Basic Research Foundation (2019B1515130001),  Key R\&D Program of Guangdong Province (2018B010107005). Jian Yin and Huishuai Zhang are corresponding authors.